\documentclass{article}
\usepackage{ijcai26}

\usepackage{times}
\usepackage{soul}
\usepackage{url}
\usepackage[hidelinks]{hyperref}
\usepackage[utf8]{inputenc}
\usepackage[small]{caption}
\usepackage{graphicx}
\usepackage{amsmath}
\usepackage{amsthm}
\usepackage{booktabs}
\usepackage{algorithm}
\usepackage{algorithmic}
\usepackage[switch]{lineno}
\usepackage{multirow}
\usepackage{makecell}
\usepackage{needspace}

\usepackage{amsmath,amssymb,amsthm}
\usepackage{mathtools}
\usepackage{enumitem}
\usepackage{xcolor}
\usepackage{amssymb}
\usepackage{array}

\usepackage{graphicx}   % for \resizebox
\usepackage{subcaption} % for subfigure
\usepackage{tikz}
\usetikzlibrary{calc}

\usepackage[dvipsnames]{xcolor}

\definecolor{lightred}{RGB}{241, 225, 222}
\definecolor{color1}{rgb}{0.1,0.498039215686275,0.9549019607843137}
\definecolor{alizarin}{rgb}{0.82, 0.1, 0.26}
\definecolor{antiquewhite}{rgb}{0.98, 0.92, 0.84}
\definecolor{azure}{rgb}{0.94, 1.0, 1.0}
\definecolor{offwhite}{rgb}{0.98, 0.95, 0.95}
\definecolor{pigment}{rgb}{0.2, 0.2, 0.6}
\definecolor{darkgray}{gray}{0.35} % tweak 0.25–0.5 to taste

\newtheorem{definition}{Definition}
\newtheorem{lemma}{Lemma}

\newcommand{\PreAdv}{\mathsf{PreC}}

\newcommand{\downcl}{\downarrow \!}

\newcommand{\prem}{\mathsf{PreMC}}
\newcommand{\premm}{\mathsf{PreMMC}}
\newcommand{\maxc}{\mathsf{Max}}

\newcommand{\ltlf}{\text{LTL}_f}

\newcommand{\ltlU}{\,\textsf{U}\,}
\newcommand{\ltlX}{\textsf{X}\,}

\newcommand{\ltltrue}{\text{true}}
\newcommand{\ltlfalse}{\text{false}}
\newcommand{\ap}{\textsf{AP}}

\newcommand{\aut}{\mathcal{G}}
\newcommand{\arena}{\mathcal{G}}

\newcommand{\y}{\mathcal{Y}}
\newcommand{\x}{\mathcal{X}}
\newcommand{\e}{\mathcal{E}}
\newcommand{\Win}{\mathsf{Win}}
\newcommand{\WinM}{\mathsf{WinM}}
\newcommand{\WinMM}{\mathsf{WinMM}}

\newcommand{\poly}{\text{poly}}

\usepackage{todonotes}

\definecolor{softpeach}{RGB}{249,200,186}

\newtheorem{example}{Example}
\newtheorem{theorem}{Theorem}

\title{Multi-Property Synthesis}

\author{
Christoph Weinhuber$^{1,}$\thanks{Authors contributed equally.}
\and
Yannik Schnitzer$^{1,*}$
\and \\
Alessandro Abate$^{1}$
\and
David Parker$^{1}$
\and 
Giuseppe De Giacomo$^{1}$
\and
Moshe Y.~Vardi$^{2}$\\
\affiliations
$^{1}$University of Oxford, Oxford, UK\\
$^{2}$Rice University, Houston, Texas, US\\
}

\begin{document}

\maketitle

\begingroup
\renewcommand{\thefootnote}{\arabic{footnote}}
\setcounter{footnote}{0}
\footnotetext[1]{firstname.lastname@cs.ox.ac.uk, \textsuperscript{2} vardi@cs.rice.edu}
\endgroup

\begin{abstract}
We study $\ltlf$ synthesis with multiple properties, where satisfying all properties may be impossible. Instead of enumerating subsets of properties, we compute in one fixed-point computation the relation between product-game states and the goal sets that are realizable from them, and we synthesize strategies achieving maximal realizable sets. We develop a fully symbolic algorithm that introduces Boolean goal variables and exploits monotonicity to represent exponentially many goal combinations compactly. Our approach substantially outperforms enumeration-based baselines, with speedups of up to two orders of magnitude.
\end{abstract}

\section{Introduction}
The problem of synthesis from temporal specifications lies at the core of many decision-making problems in artificial intelligence, most prominently automated planning and controller synthesis. In this setting, we automatically construct an agent strategy that guarantees the satisfaction of a given temporal specification against all admissible behaviours of the environment. Linear Temporal Logic (LTL) has long served as the standard formalism for expressing temporally extended goals and assumptions in infinite-horizon settings, particularly in reactive synthesis~\cite{DBLP:conf/icalp/PnueliR89}.

In contrast, many real-world AI tasks are inherently \emph{finite-horizon}: plans are meant to terminate, tasks have a clear notion of completion, and success is achieved once a desired condition is reached. The finite-trace variant of LTL, denoted $\ltlf$, has emerged as a natural and expressive formalism for specifying goals in such domains. Over the past years, $\ltlf$ synthesis has matured into a standard tool for generating correct-by-construction behaviours for finite-horizon tasks~\cite{DBLP:conf/ijcai/GiacomoV13,DBLP:conf/ijcai/GiacomoV15}. Typical application areas include classical and contingent planning~\cite{DBLP:conf/ijcai/GiacomoR18,DBLP:conf/aips/CamachoBMM18}, service composition~\cite{DBLP:journals/tweb/MontaliPACMS10}, and business process modelling~\cite{DBLP:conf/bpm/PesicA06,DBLP:conf/bpm/GiacomoMGMM14}, where the objective is not to react indefinitely, but to reach a satisfactory outcome and stop.

While existing $\ltlf$ synthesis techniques are well-suited for reasoning about single, monolithic specifications, they are largely constrained to an \emph{all-or-nothing} paradigm: the procedure either successfully synthesises a strategy that satisfies the entire specification, or yields no strategy at all, declaring the problem unsolvable \cite{DuretLutzZPGV25}.

In many realistic scenarios, however, such as robotics, multi-service orchestration, or planning under resource constraints~\cite{DBLP:conf/aaai/BrielNDK04,DBLP:journals/aim/BaierM08}, agents face \emph{over-subscription}~\cite{DBLP:conf/aips/Smith04}. That is, the specification naturally decomposes into multiple properties (or goals), but satisfying all of them simultaneously may be impossible due to conflicting requirements, limited resources, or adversarial environmental choices.

In such settings, the all-or-nothing paradigm is overly restrictive, as failure is rarely binary. Even when the conjunction of all goals is unrealisable, an agent may still be able to guarantee the satisfaction of some meaningful subset of them. 
We study the fundamental question:
If not all goals can be achieved simultaneously, which subsets can, and which of them are most desirable? 
Answering this question is crucial for enabling informed trade-offs, and principled decision-making in over-subscribed domains.

A straightforward way to address this question is to enumerate all subsets of goals and, for each subset, solve the synthesis problem for the conjunction of its elements~\cite{DBLP:conf/sigmod/AgrawalIS93}. This, however, can become infeasible even for a moderate number of properties, as it requires solving an exponential number of synthesis instances. 

We propose a multi-property synthesis technique that reasons about all goal combinations simultaneously. Our approach computes, via a \emph{single} fixed-point computation, the full relation between states and the subsets of goals that can be simultaneously achieved from them. Rather than checking realisability one subset at a time, we compute, in one pass, which combinations of goals can be guaranteed against all admissible environment behaviours.

Crucially, our method is not limited to deciding which subsets of goals \emph{can} be simultaneously achieved; it also synthesises a strategy that achieves them.
The agent interacts with the environment until it reaches a state where the most desirable set of realisable goals has been achieved, at which point it can stop. 
Strategies are constructed not to react indefinitely, but to ensure progress toward satisfying the desired goals and to halt when this objective has been met.

Beyond the fixed-point characterisation, we develop a fully symbolic technique to the multi-property computation.
We leverage Boolean synthesis~\cite{DBLP:conf/cav/FriedTV16} to perform the fixed-point computation in a compact symbolic representation, which is often dramatically more efficient than explicit constructions~\cite{DBLP:conf/ijcai/ZhuTLPV17}.
Notice that, although the problem domain spans all subsets of goals, it is structured by monotonicity -- achieving a set of goals implies achieving any subset of them.
Our symbolic encoding is designed to exploit this structure directly, allowing a single Boolean representation to compactly characterize exponentially many goal combinations without explicitly enumerating them.

In summary, this paper makes the following contributions:
\begin{enumerate}
    \item We introduce a technique for multi-property $\ltlf$ synthesis, capturing the realizability of arbitrary subsets of goals in over-subscribed domains.
    \item We present a fixed-point characterization that computes, in a single construction, all maximal realizable goal sets and the strategies that achieve them.
    \item We develop a symbolic synthesis algorithm that captures the exponentially-many goal combinations in a single compact Boolean form.
    \item We empirically demonstrate that our approach significantly outperforms current state-of-the-art enumeration-based approaches by up to two orders of magnitude.
\end{enumerate}

\section{Related Work}

\paragraph{Multi-property Verification.}
Analysing multiple properties over a fixed model has also been studied in the verification literature.
\citeauthor{DBLP:conf/date/CabodiN11}~(\citeyear{DBLP:conf/date/CabodiN11}) consider checking multiple invariants on the same circuit and show how to reduce repeated work by exploiting shared structure across properties.
Since verification queries are \emph{independent}, one property being true does not come at the expense of another, and the main challenge is amortizing repeated proof effort.
In our synthesis setting, by contrast, properties can \emph{compete}: enforcing one objective may restrict the available behaviors and thereby preclude satisfying another.

In probabilistic verification, multi-objective model checking for Markov decision processes (MDPs) asks whether there exists a strategy that simultaneously meets bounds on the individual probabilities of satisfying multiple properties, and supports trade-off analysis by computing Pareto fronts~\cite{DBLP:conf/tacas/EtessamiKVY07}.
Our multi-property synthesis can be seen as an analogous trade-off analysis in a qualitative, adversarial setting: instead of a continuous Pareto front over probability thresholds, we obtain a \emph{discrete} frontier given by the maximal realizable sets of properties.

\paragraph{Maximum-realizability Synthesis}
\citeauthor{DBLP:conf/atva/DimitrovaGT18}(\citeyear{DBLP:conf/atva/DimitrovaGT18}) study \emph{infinite-trace} LTL synthesis with hard specifications and soft safety requirements, of which potentially not all can be satisfied on a single execution.
They employ a multi-valued quantitative semantics grading how well each soft property is satisfied along infinite executions that meet the hard constraint~\cite{DBLP:journals/acta/DimitrovaGT20}.
Their technique is optimisation-driven: building on bounded synthesis, they search for a controller that maximises this quantitative score within a bound on controller size~\cite{DBLP:conf/atva/DimitrovaGT18}.

In our setting, we are given many $\ltlf$ goals and compute which subsets are simultaneously realizable, together with strategies fulfill them.
We exploit that $\ltlf$ synthesis reduces to DFA reachability games~\cite{DBLP:conf/ijcai/GiacomoV15}, where satisfaction amounts to reaching accepting states.
This yields a direct fixed-point computation with a local acceptance base case, computing in one construction the realizability relation for all states and goal combinations.

\paragraph{Over-subscription and Planning with Preferences.}
Over-subscription and partial-satisfaction planning address settings where limitations prevent fulfilling all goals by utilizing heuristics or utility functions to select valuable subsets~\cite{DBLP:conf/aips/Smith04,DBLP:conf/aaai/BrielNDK04}, or by optimizing user-specified preferences~\cite{DBLP:journals/aim/BaierM08,DBLP:conf/kr/BienvenuFM06,DBLP:conf/ijcai/SohrabiBM09}. 
However, these approaches produce a single plan for \emph{non-adversarial} models. 
In contrast, we synthesise strategies that fulfill goal sets against all admissible environment behaviours.

\section{Preliminaries}
\label{sec:prelims}
In this section, we review the basics of Linear Temporal Logic over finite traces and the standard reduction of synthesis to DFA reachability games.

\subsection{Linear Temporal Logic over finite traces}
Linear Temporal Logic over finite traces ($\ltlf$) shares the same syntax as LTL~\cite{DBLP:conf/focs/Pnueli77}. Given a set of atomic propositions $\ap$, the syntax of $\ltlf$ formulas is defined as:
$$\varphi ::=  \ltltrue \mid  a \in \ap \mid \neg \varphi \mid \varphi \land \varphi \mid \varphi \mid \ltlX \varphi \mid \varphi \ltlU \varphi.$$

A trace is a \emph{finite} sequence $\rho = \rho[0]\rho[1]\cdots \rho[m]$, where each $\rho[i] \in 2^{\ap}$ specifies which atomic propositions hold at time $i$. We write $|\rho| := m+1$ for the length of $\rho$.
We use the satisfaction relation $\rho,i \models \varphi$ to denote that $\varphi$ holds at position $i$ of $\rho$. The satisfaction relation is defined by:
\[
\begin{array}{ll}
\rho, i \models a               \!& \iff a \in \rho[i], \\
\rho, i \models \neg \varphi         \!& \iff \rho, i \not\models \varphi, \\
\rho, i \models \varphi_1 \wedge \varphi_2         \!& \iff \rho, i \models \varphi_1 \text{~and~} \rho, i \models \varphi_2, \\
\rho, i \models \ltlX \varphi         \!& \iff i + 1 < |\rho| \text{~and~} \rho, i + 1 \models \varphi, \\
\rho, i \models \varphi_1 \ltlU \varphi_2         \!& \iff \exists j.~ i \le j < |\rho| \text{~s.t.~} \rho, j \models \varphi_2 \\
\!& \qquad ~~\text{and~} \forall k.~ i \le k < j,~ \rho, k \models \varphi_1.
\end{array} 
\]
A trace $\rho$ satisfies $\varphi$, denoted $\rho \models \varphi$, iff $\rho,0 \models \varphi$.

\subsection{$\ltlf$ Synthesis and DFA Games}
Synthesis for $\ltlf$ properties is the process of computing an agent strategy that enforces a temporal specification against all admissible behaviours of its environment.
We partition the atomic propositions $\ap$ into two disjoint sets: \emph{inputs} $\x$ controlled by the environment and \emph{outputs} $\y$ controlled by the agent.
An execution is a finite trace $\rho$ over $2^{\x \cup \y}$ that is generated as follows: at each time $i$, the agent first chooses a valuation $Y_i \in 2^\y$, after which the environment responds with a valuation $X_i \in 2^{\x}$.
The resulting combined valuation $(X_i,Y_i)$ forms the $i$-th position $\rho[i]$ of the trace.

An \emph{agent strategy} is a function
$g : (2^\x)^* \to 2^\y$
that maps the (potentially empty) history of observed inputs $X_0 X_1 \cdots X_{i-1}$ to the next output choice $Y_{i} = g(X_0 \cdots X_{i-1})$.
The \emph{realizability} problem asks whether there exists such a strategy $g$ such that every finite trace consistent with $g$ satisfies $\varphi$.
The \emph{synthesis} problem additionally requires constructing a witnessing \emph{winning} strategy. Since $\ltlf$ is interpreted over finite traces, a winning strategy must ensure that the play reaches a position $i$ for which the trace $\rho = (X_0 \cup Y_0)\dots(X_i \cup Y_i)$ satisfies $\varphi$ and can then terminate.

\begin{definition}[$\ltlf$ Synthesis]\label{def:ltlf-synthesis}
Let $\varphi$ be an $\ltlf$ formula over atomic propositions $\ap = \x \cup \y$.
Given an infinite input sequence $\xi = X_0 X_1 X_2 \cdots \in (2^\x)^{\omega}$ and a strategy $g$, the induced finite trace $\rho_k(g,\xi)$ up to horizon $k$ is given as:
\[
(g(\epsilon) \cup X_0)(g(X_0) \cup X_1)\cdots (g(X_0\cdots X_{k-1}) \cup X_k ).
\]
We say that $\varphi$ is \emph{realizable} if there exists a strategy $g$ such that for any $\xi \in (2^\x)^{\omega}$ there is some $k\ge 0$ with
$
\rho_k(g,\xi) \models \varphi .
$
The \emph{synthesis} problem is to construct such a winning strategy.
\end{definition}
Notice that, unlike in LTL reactive synthesis over infinite traces, agent and environment do not interact indefinitely: a winning strategy ensures that the interaction generates a finite trace that satisfies $\varphi$ and then chooses to stop at that point.

The standard approach to $\ltlf$ synthesis is to solve the corresponding \emph{DFA game}~\cite{DBLP:conf/ijcai/GiacomoV13}.
Given an $\ltlf$ formula $\varphi$ over $\ap = \x \cup \y$, we can construct a DFA $\aut_\varphi$ over alphabet $2^{\x \cup \y}$ that accepts exactly the finite traces satisfying $\varphi$.
Synthesis then reduces to solving a reachability game played on $\aut_\varphi$, also called the \emph{game arena}.

\begin{definition}[DFA Game Arena]\label{def:dfa-game}
A DFA game arena is a tuple $\arena=(2^{\x \cup \y},S,s_0,\delta,F)$, where $2^{\x \cup \y}$ is the alphabet of the game, $S$ is a finite set of states, $s_0\in S$ is the initial state,
$\delta : S \times 2^{\x \cup \y} \to S$ is the transition function, and $F \subseteq S$ is a set of final states.
\end{definition}

A play of the DFA game proceeds in rounds: at round $i$ the game is in state $s_i \in S$ and the agent chooses an output valuation $Y_i\in 2^{\y}$, then the environment chooses an input valuation $X_i\in 2^\x$, and the game moves to
$s_{i+1} := \delta(s_i, Y_i \cup X_i)$.
The agent wins a play if it reaches final state belonging to $F$, at which point the play can be considered terminated.

Since $\aut_\varphi$ recognizes exactly the finite traces satisfying $\varphi$, reaching a final state $F$ in the automaton is equivalent to having generated a finite prefix that satisfies $\varphi$.
Thus, an agent strategy $g : (2^\x)^* \to 2^\y$ is winning in the DFA game iff, for every infinite input sequence $\xi \in (2^\x)^\omega$, the induced play visits some state in $F$ after finitely many rounds.

To characterize from which states the agent can force a win, we use the standard \emph{controllable predecessor} operator~\cite{DBLP:conf/ijcai/GiacomoV15}:
for a set $\e \subseteq S$, we define
\[
\PreAdv(\e) := \{s\in S \mid \exists Y\in 2^{\y}.\forall X\in 2^{\x}. \delta(s, Y \cup X)\in \e\}.
\]
Intuitively, $\PreAdv(\e)$ contains the states from which the agent has a one-step move that guarantees reaching $\e$ regardless of the environment.
The set of winning states is the least fixed point of repeatedly adding such predecessors, starting from the accepting states:
\[
\Win_0 := F
\quad\text{and}\quad
\Win_{i+1} := \Win_i \cup \PreAdv(\Win_i).
\]
As $S$ is finite, the sequence stabilizes after at most $|S|$ iterations in a fixed point $\Win := \bigcup_{i\ge 0} \Win_i$, yielding a linear-time computation in the number of game states.
The agent has a winning strategy iff $s_0\in \Win$.

Once the winning region $\Win$ has been computed, we can extract a concrete winning strategy by choosing, for each winning state $s\in \Win$, an appropriate agent output $Y \in 2^{\y}$.
Formally, we choose a function $\omega : \Win \to 2^{\y}$ such that for every $s\in \Win_i$ with $i > 0$,
\[
\forall X \in 2^{\x}.\; \delta(s, \omega(s) \cup X ) \in \Win_{i-1}.
\]
Such a choice exists by construction of $\Win$~\cite{DBLP:conf/ijcai/GiacomoV13}. For states $s \in \Win_0 = F$, the agent chooses a fresh \emph{done} action, indicating that it terminates the play.
Winning strategies can be formulated as deterministic finite-state controllers in the standard form of a transducer:
\[
T = (2^{\x \cup \y}, Q, q_0, \eta, \omega),
\]
where $Q := \Win$, $q_0 := s_0$, the output function is the chosen $\omega$, and the transition function is
$
\eta(q,X) := \delta(q, \omega(q) \cup X).
$
For an input sequence $X_0X_1\cdots$, the transducer outputs $Y_i := \omega(q_i)$ and updates $q_{i+1} := \eta(q_i,X_i)$, generating a winning play in the game $G$. The transducer therefore implements a concrete winning strategy.

\section{Multi-property $\ltlf$ Synthesis}
\label{sec:multi-prop}

We lift the standard notion of synthesis from single properties to sets of properties. We aim to construct a single strategy that forces the play to generate a trace satisfying multiple properties against all input behaviors.

\begin{definition}[Multi-property $\ltlf$ Synthesis]\label{def:multigoal-real}
Let $\Phi=\{\varphi_1,\dots,\varphi_n\}$ be set of $n$ $\ltlf$ properties (or \emph{goals}) over the same atomic propositions $\ap=\x \cup \y$.
A goal set $C\subseteq \Phi$ is \emph{realizable} if there exists a strategy
$g:(2^\x)^{*}\to 2^\y$ such that for every input sequence $\xi=X_0X_1\cdots \in (2^\x)^{\omega}$, there is a $k \geq 0$ such that the induced finite trace $\rho_k(g,\xi) \vDash \varphi_i$, forall $\varphi_i \in C$. 
%satisfies all properties in $C$. 
The \emph{multi-property synthesis} problem is to construct such a winning strategy.

\end{definition}

We extend the DFA-game approach to $\ltlf$ synthesis from a single property to a set of properties, by constructing a single arena that tracks progress towards all goals simultaneously.
We consider a fixed set of $n$ $\ltlf$ properties (or \emph{goals}) $\Phi=\{\varphi_1,\dots,\varphi_n\}$ over the same atomic propositions $\ap=\x \cup \y$.
For each goal $\varphi_i \in \Phi$, let
$
\aut_i=(2^{\x \cup \y},\,S_i,\,q_i^0,\,\delta_i,\,F_i)
$
be a DFA recognizing exactly the finite traces satisfying $\varphi_i$.

\begin{samepage}
\begin{definition}[Multi-property Game Arena]\label{def:product-arena}
The \emph{multi-property arena} of $\Phi$ is the product:
\[
\arena^{\times}=(2^{\x\cup \y},\,S^\times,\,s^\times_0,\,\delta^{\times}),
\]
where $S^\times:=\bigtimes_{i=1}^n S_i$ is the set of product states, $s_0^\times:=(q_1^0,\dots,q_n^0)$ is the initial product state, and the transition function is defined componentwise by
\[
\delta^{\times}\bigl((q_1,\dots,q_n),a\bigr)
\;:=\;
\bigl(\delta_1(q_1,a),\dots,\delta_n(q_n,a)\bigr).
\]
\end{definition}
Notice that $\arena^\times$ does not come with a designated set of final states: in the multi-property setting, there is no single notion of completion, since a product state may witness the satisfaction of some goals but not others.
\end{samepage}
The enumeration-based approach fixes a target conjunction for some subset $\Phi' \subseteq \Phi$, which induces a corresponding acceptance condition, and then solves a separate synthesis game for each choice of $\Phi'$, rebuilding product arenas an exponential number of times.
In contrast, since we aim to compute \emph{all} fulfillable subsets simultaneously, we do not restrict $\arena^\times$ to a single acceptance condition.
Instead, in the following we develop a multi-property synthesis procedure that works on the \emph{single} product arena $\arena^\times$ and simultaneously determines, for all $\Phi' \subseteq \Phi$, which goal combinations are fulfillable from each product state and synthesises strategies that fulfill them.

A state $s=(q_1,\dots,q_n)\in S^\times$ consists of the component states of the individual automata, one for each property $\varphi_i$.
Accordingly, a set of goals $C\subseteq \Phi$ is \emph{simultaneously satisfied} at $s$ if all corresponding components are accepting.

\begin{definition}[Multi-Property Satisfaction]\label{def:goals-satisfied}
Let $s=(q_1,\dots,q_n)\in S^\times$ and let $C\subseteq \Phi$.
We write $s \models C$ if for every property $\varphi_i\in C$ we have $q_i \in F_i$.
\end{definition}

A goal set $C\subseteq \Phi$ is realizable from a state $s\in S^\times$ if there exists a strategy
$g:(2^\x)^{*}\to 2^\y$
such that for every input sequence $\xi=X_0X_1\cdots \in (2^\x)^{\omega}$, the play
$s_0:=s$ and $s_{t+1}:=\delta^\times(s_t, g(X_0\cdots X_{t-1}) \cup X_t )$
reaches some $k$ with $s_k \models C$.
At that point the strategy can stop the play, and since each $\aut_i$ recognizes the traces satisfying $\varphi_i$, reaching a product state $s_k \in S^\times$ with $s_k \models C$ means that the induced finite trace is accepted by each automaton $\aut_i$ for $\varphi_i\in C$, hence satisfies all goals in $C$ simultaneously.

\subsection{Solving Multi-property Games}
\label{sec:solving-multi-property-games}

We introduce our approach for computing a \emph{multi-property winning relation} over pairs $(s,C)$, where $s\in S^\times$ and $C\subseteq \Phi$.
This relation characterizes, for each product state $s$, exactly which goal sets $C$ are realizable, and allows us to extract a winning strategy for any such $C$.
As for single properties, it is obtained by a fixed-point computation that lifts winning regions from states to \emph{state--property-set pairs}: the base consists of those $(s,C)$ where all goals in $C$ are already satisfied at $s$, and otherwise the agent must choose an output $Y\in 2^\y$ such that for all $X\in 2^\x$ the agent can enforce realizing all properties in $C$ simultaneously from any successor state. This propagates realizability for all goal sets across all states.

For a relation $\mathcal{E}\subseteq S^\times\times 2^{\Phi}$ we define the \emph{controllable multi-property predecessor} by:
\[
\begin{aligned}
\prem(\mathcal{E}) := \{(s,C)\mid {}&
\exists Y\in 2^{\y}\ \forall X\in 2^{\x}.\\
&(\delta^\times(s, Y \cup X),C)\in\mathcal{E}\}.
\end{aligned}
\]
Intuitively, $\prem(\mathcal{E})$ contains the pairs $(s,C)$ from which the agent has a one-step move that guarantees preserving realizability of the same goal set $C$, regardless of the environment's input choice.
The multi-property winning relation is then obtained as the least fixed point of repeatedly adding such predecessors, starting from the base pairs where $C$ already holds in the current product state:
\[
\begin{aligned}
\WinM_0 &:= \{(s,C)\mid s\models C\},\\
\WinM_{i+1} &:= \WinM_i \cup \prem(\WinM_i).
\end{aligned}
\]
Notice that $\WinM_0$ can be computed in linear time: since checking $s\models C$ only amounts to inspecting which components of a state $s = (q_1, \dots, q_n)$ are accepting $q_i\in F_i$, and thus which goal sets $C$ are already satisfied at $s$.
As $S^\times$ and $\Phi$ are finite, the sequence stabilizes after at most $|S^\times| \cdot 2^{|\Phi|}$ iterations in a fixed point
$\WinM := \bigcup_{i\ge 0} \WinM_i$.

\begin{theorem}\label{thm:WM-correct}
The fixed point $\WinM$ characterizes multi-property realizability: for all $s\in S^\times$ and $C\subseteq \Phi$,
\[
(s,C)\in \WinM \quad \Longleftrightarrow \quad C \text{ is realizable from } s.
\]
\end{theorem}

\begin{figure*}[t]
  \centering
    \resizebox{0.93\textwidth}{!}{
  % --- Panel (a)
  \begin{subfigure}[t]{0.29\textwidth}
    \centering
    \resizebox{\linewidth}{!}{%
      \begin{tikzpicture}[->, thick, >=latex, scale=0.6, transform shape]
        % Freeze bounding box so all panels scale identically under \resizebox
        \path[use as bounding box] (-5.4,-2) rectangle (4.4,7.6);

        \tikzstyle{state}=[circle, draw, minimum size=1.2cm, align=center,
          fill=lightred, fill opacity=0.8, text opacity=1]
        \tikzstyle{synth}=[draw, rectangle, minimum size=1.05cm, align=center,
          fill=BrickRed, fill opacity=0.455, text opacity=1]
        \tikzstyle{edgelab}=[midway, fill=white, inner sep=2pt, text opacity=1]

        \node[state] (s0) at (0,0) {\huge$\boldsymbol{s_0}$};

        \node[synth] (qL) at (-2.2,2.9) {};
        \node[synth] (qR) at ( 2.2,2.9) {};

        % s1 and s2 moved slightly closer together
        \node[state] (s1) at (-3.9,5.8) {\huge$\boldsymbol{s_1}$};
        \node[state] (s2) at (-0.5,5.8) {\huge$\boldsymbol{s_2}$};

        \node[state] (s3) at (3.2,5.8) {\huge$\boldsymbol{s_3}$};

        % --- extra labels (panel a only) -- bold, \LARGE; split across two lines
        \node[font=\Large, align=center] at ($(s1.north)+(-0.75,0.9)$)
          {$\boldsymbol{\{\{\varphi_1,\varphi_2,\varphi_3\},\{\varphi_1,\varphi_2\},}$\\
           $\boldsymbol{\{\varphi_2,\varphi_3\},\dots,\emptyset\}}$};

        \node[font=\Large, align=center] at ($(s2.north)+(0.25,0.9)$)
          {$\boldsymbol{\{\{\varphi_1,\varphi_2\},\{\varphi_1\},}$\\
           $\boldsymbol{\{\varphi_2\},\emptyset\}}$};

        \node[font=\Large, align=center] at ($(s3.north)+(0.25,0.9)$)
          {$\boldsymbol{\{\{\varphi_2,\varphi_3\},\{\varphi_2\},}$\\
           $\boldsymbol{\{\varphi_3\},\emptyset\}}$};

        \node[font=\Large, anchor=north east, text=darkgray, align=center]
          at ($(qL.south west)+(+0.15,-0.10)$)
          {$\boldsymbol{\{\{\varphi_1,\varphi_2\},\{\varphi_1\},}$\\
           $\boldsymbol{\{\varphi_2\},\emptyset\}}$};

        \node[font=\Large, anchor=north west, text=darkgray, align=center]
          at ($(qR.south east)+(-0.15,-0.10)$)
          {$\boldsymbol{\{\{\varphi_2,\varphi_3\},\{\varphi_2\},}$\\
           $\boldsymbol{\{\varphi_3\},\emptyset\}}$};

        \node[font=\Large, align=center] at ($(s0.south)+(0,-0.8)$)
          {$\boldsymbol{\{\{\varphi_1,\varphi_2\},\{\varphi_2,\varphi_3\},}$\\
           $\boldsymbol{\{\varphi_1\},\dots,\emptyset\}}$};

        \path (s0) edge[bend left=15]  node[edgelab] {\huge$y_1$} (qL);
        \path (s0) edge[bend right=15] node[edgelab] {\huge$y_2$} (qR);

        \path (qL) edge[bend left=15]  node[edgelab] {\huge$x_1$} (s1);
        \path (qL) edge[bend right=15] node[edgelab] {\huge$x_2$} (s2);

        \path (qR) edge node[edgelab] {\huge$x_1,x_2$} (s3);
      \end{tikzpicture}%
    }
    \caption{Multi-Property.}
    \label{fig:synthgame-a}
  \end{subfigure}\hspace{0.04\textwidth}
  % --- Panel (b)
  \begin{subfigure}[t]{0.29\textwidth}
    \centering
    \resizebox{\linewidth}{!}{%
      \begin{tikzpicture}[->, thick, >=latex, scale=0.6, transform shape]
        % Freeze bounding box so all panels scale identically under \resizebox
        \path[use as bounding box] (-5.4,-2) rectangle (4.4,7.6);

        \tikzstyle{state}=[circle, draw, minimum size=1.2cm, align=center,
          fill=lightred, fill opacity=0.8, text opacity=1]
        \tikzstyle{synth}=[draw, rectangle, minimum size=1cm, align=center,
          fill=BrickRed, fill opacity=0.455, text opacity=1]
        \tikzstyle{edgelab}=[midway, fill=white, inner sep=2pt, text opacity=1]

        \node[state] (s0) at (0,0) {\huge$\boldsymbol{s_0}$};

        \node[synth] (qL) at (-2.2,2.9) {};
        \node[synth] (qR) at ( 2.2,2.9) {};

        % s1 and s2 moved slightly closer together
        \node[state] (s1) at (-3.9,5.8) {\huge$\boldsymbol{s_1}$};
        \node[state] (s2) at (-0.5,5.8) {\huge$\boldsymbol{s_2}$};

        \node[state] (s3) at (3.2,5.8) {\huge$\boldsymbol{s_3}$};

        % --- extra labels (panel b only) -- bold
        \node[font=\LARGE] at ($(s1.north)+(0,0.55)$) {$\boldsymbol{\{\varphi_1,\varphi_2,\varphi_3\}}$};
        \node[font=\LARGE] at ($(s2.north)+(0,0.55)$) {$\boldsymbol{\{\varphi_1,\varphi_2\}}$};
        \node[font=\LARGE] at ($(s3.north)+(0,0.55)$) {$\boldsymbol{\{\varphi_2,\varphi_3\}}$};

        \node[font=\LARGE, anchor=north east, text=darkgray] at ($(qL.south west)+(+0.05,-0.10)$) {$\boldsymbol{\{\varphi_1,\varphi_2\}}$};
        \node[font=\LARGE, anchor=north west, text=darkgray] at ($(qR.south east)+( -0.05,-0.10)$) {$\boldsymbol{\{\varphi_2,\varphi_3\}}$};

        \node[font=\LARGE] at ($(s0.south)+(0,-0.65)$) {$\boldsymbol{\{\{\varphi_1,\varphi_2\},\{\varphi_2,\varphi_3\}\}}$};

        \path (s0) edge[bend left=15]  node[edgelab] {\huge$y_1$} (qL);
        \path (s0) edge[bend right=15] node[edgelab] {\huge$y_2$} (qR);

        \path (qL) edge[bend left=15]  node[edgelab] {\huge$x_1$} (s1);
        \path (qL) edge[bend right=15] node[edgelab] {\huge$x_2$} (s2);

        \path (qR) edge node[edgelab] {\huge$x_1,x_2$} (s3);
      \end{tikzpicture}%
    }
    \caption{Maximal Multi-Property.}
    \label{fig:synthgame-b}
  \end{subfigure}\hspace{0.04\textwidth}
  % --- Panel (c)
  \begin{subfigure}[t]{0.29\textwidth}
    \centering
    \resizebox{\linewidth}{!}{%
      \begin{tikzpicture}[->, thick, >=latex, scale=0.6, transform shape]
        % Freeze bounding box so all panels scale identically under \resizebox
        \path[use as bounding box] (-5.4,-2) rectangle (4.4,7.6);

        \tikzstyle{state}=[circle, draw, minimum size=1.2cm, align=center,
          fill=lightred, fill opacity=0.8, text opacity=1]
        \tikzstyle{synth}=[draw, rectangle, minimum size=1.05cm, align=center,
          fill=BrickRed, fill opacity=0.455, text opacity=1]
        \tikzstyle{edgelab}=[midway, fill=white, inner sep=2pt, text opacity=1]

        \node[state] (s0) at (0,0) {\huge$\boldsymbol{s_0}$};

        \node[synth] (qL) at (-2.2,2.9) {};
        \node[synth] (qR) at ( 2.2,2.9) {};

        % s1 and s2 moved slightly closer together
        \node[state] (s1) at (-3.9,5.8) {\huge$\boldsymbol{s_1}$};
        \node[state] (s2) at (-0.5,5.8) {\huge$\boldsymbol{s_2}$};

        \node[state] (s3) at (3.2,5.8) {\huge$\boldsymbol{s_3}$};

        % --- extra labels (panel c only) -- bold, \LARGE
        \node[font=\LARGE] at ($(s1.north)+(0,0.55)$) {$\boldsymbol{\top}$};
        \node[font=\LARGE] at ($(s2.north)+(0,0.55)$) {$\boldsymbol{\neg k_3}$};
        \node[font=\LARGE] at ($(s3.north)+(0,0.55)$) {$\boldsymbol{\neg k_1}$};

        \node[font=\LARGE, anchor=north east, text= darkgray] at ($(qL.south west)+(-0.15,-0.10)$) {$\boldsymbol{\neg k_3}$};
        \node[font=\LARGE, anchor=north west, text= darkgray] at ($(qR.south east)+( 0.15,-0.10)$) {$\boldsymbol{\neg k_1}$};

        \node[font=\LARGE] at ($(s0.south)+(0,-0.65)$) {$\boldsymbol{\neg k_1 \,\lor\, \neg k_3}$};

        \path (s0) edge[bend left=15]  node[edgelab] {\huge$y_1$} (qL);
        \path (s0) edge[bend right=15] node[edgelab] {\huge$y_2$} (qR);

        \path (qL) edge[bend left=15]  node[edgelab] {\huge$x_1$} (s1);
        \path (qL) edge[bend right=15] node[edgelab] {\huge$x_2$} (s2);

        \path (qR) edge node[edgelab] {\huge$x_1,x_2$} (s3);
      \end{tikzpicture}%
    }
    \caption{Symbolic.}
    \label{fig:synthgame-c}
  \end{subfigure}
  }

  \caption{Illustration of our synthesis procedures on a problem with $2^\y=\{y_1,y_2\}$, $2^\x=\{x_1,x_2\}$, and properties $\Phi=\{\varphi_1,\varphi_2,\varphi_3\}$.
(a) and (b) depict the resulting winning relation $\WinM$ projected onto states: each state is annotated with the goal sets that are realizable from it.
(a) shows the explicit multi-property fixed point that tracks all realizable goal sets (Section~\ref{sec:solving-multi-property-games}) and 
(b) shows the maximal variant (Section~\ref{sec:maximal_multi-prop}).
(c) illustrates the symbolic variant (Section~\ref{sec:symbolic}), where the winning relation is represented compactly as a Boolean formula.}

  \label{fig:synthgame-3panel}
\end{figure*}
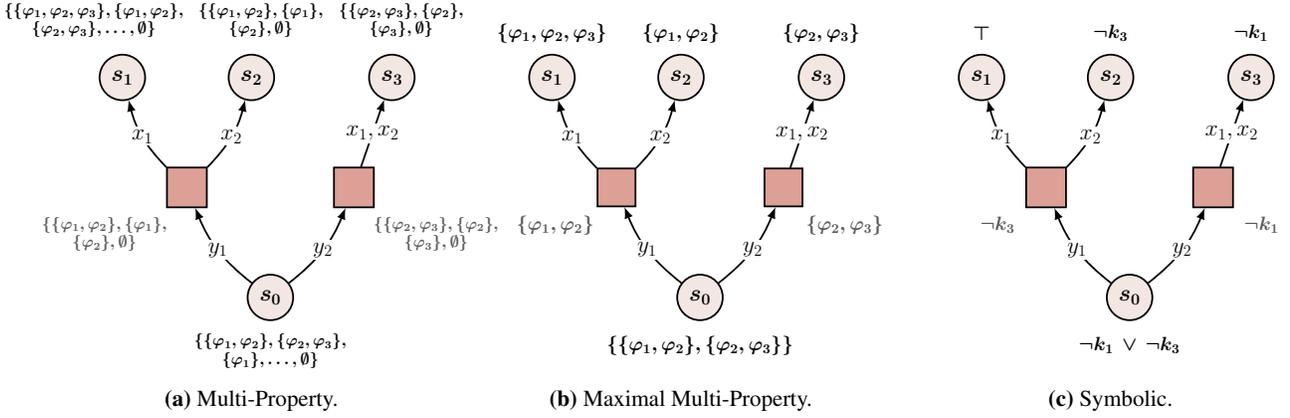

The relation $\WinM$ generalizes the traditional winning region: for any goal set $C\subseteq \Phi$, there exists a winning strategy for fulfilling $C$ iff $(s_0^\times,C)\in \WinM$.
By projecting the winning relation on to the desired set of properties $C$:
\[
\WinM_i|_{C} := \{\, s\in S^\times \mid (s,C)\in \WinM_i\,\}.
\]
We can construct the corresponding winning strategy as a transducer $T_C$, just as for the single property case discussed in Section~\ref{sec:prelims}: 
choose $\omega_C:\WinM|_C\to 2^{\y}$ such that for every $s\in \WinM_i|_C$ with $i>0$,
\[
\forall X\in 2^{\x}.\;\delta^\times(s,\omega_C(s) \cup X)\in \WinM_{i-1}|_C,
\]
and for $s\in \WinM_0|_C=\{s\mid s\models C\}$ output a designated \emph{done} action and terminate. Such a choice exists by construction of $\WinM$ and the definition of $\prem$.

\begin{theorem}\label{thm:multistrat-correct}
For any $C\subseteq \Phi$ with $(s_0^\times,C)\in \Win^{M}$, a transducer $T_C$ obtained by the construction above exists and implements a winning strategy for realizing $C$: for every input sequence $\xi\in(2^\x)^\omega$, the play induced by $T_C$ reaches a state $s \in S^\times$ with $s\models C$ after finitely many steps, and the generated finite trace satisfies every $\varphi_i\in C$.
\end{theorem}

\subsection{Maximal Multi-Property Synthesis}
\label{sec:maximal_multi-prop}
The construction above avoids explicitly solving a separate synthesis instance for each $C\subseteq \Phi$, but the relation $\WinM\subseteq S^\times\times 2^\Phi$ may still be large: in the worst case, a single state can be related to exponentially many sets of properties. Nevertheless, multi-property realizability has an inherent \emph{monotonicity}: if $(s,C)\in \WinM$, then $(s,D)\in \WinM$ for every $D\subseteq C$.
Hence, many elements of $\WinM$ are redundant and subsumed by larger realizable sets.
Therefore, it is natural to maintain, for each state $s$, only the \emph{maximal} realizable sets. %of properties.

We capture this compression with an operator that removes all pairs subsumed by a strictly larger set of properties at the same state.
For a relation $\mathcal{E}\subseteq S^\times \times 2^{\Phi}$, we define
\[
\maxc(\mathcal{E})
\,:=\,
\{ (s,C)\in \mathcal{E} \mid \neg \exists (s,D)\in \mathcal{E}\ \text{with}\ C \subset D \}.
\]
Thus, $\maxc(\mathcal{E})$ retains exactly those pairs whose set of properties $C$ set is maximal with respect to $\subseteq$ among all sets associated with the same state.
In the following, we apply $\maxc$ within the fixed-point computation so that only maximal realizable sets are maintained throughout, while all smaller realizable sets remain implicit. 
Once, however, we keep only maximal sets, the predecessor step must not require \emph{the same} set $C$ to be stored explicitly for every successor. Instead, it is sufficient that each successor state admits \emph{some} realizable set $D$ that \emph{covers} $C$, i.e., $C\subseteq D$, because realizability of $D$ implies realizability of all its subsets. Accordingly, we define the \emph{maximal multi-property predecessor} for $\mathcal{E}\subseteq S^\times\times 2^\Phi$ by
\[
\begin{aligned}
\premm(\mathcal{E}) := \{(s,C)\mid {}&
\exists Y\in 2^{\y}\ \forall X\in 2^{\x}\ \exists D\supseteq C.\\
&(\delta^\times(s,Y \cup X),D)\in \mathcal{E}\ \}.
\end{aligned}
\]
Intuitively, $\premm(\mathcal{E})$ contains exactly those pairs $(s,C)$ from which the agent can, for any environment input, force a step to a successor state where it can enforce some $D\supseteq C$, and therefore can also enforce $C$.

Instead of constructing $\WinM$ and pruning redundant goal sets afterwards, we maintain throughout only those pairs whose property set is maximal for the corresponding state.
This yields the following \emph{maximal} fixed-point iteration:
\[
\begin{aligned}
\WinM_0 &:= \maxc(\{(s,C)\mid s\models C\}),\\
\WinM_{i+1} &:= \maxc\bigl(\WinM_i \cup \premm(\WinM_i)\bigr),
\end{aligned}
\]
and let $\WinM := \bigcup_{i\ge 0}\WinM_i$ denote the resulting fixed point.
In contrast to the construction in Section~\ref{sec:solving-multi-property-games}, which may accumulate many dominated pairs, the iteration above ensures that dominated pairs are eliminated on the fly while all smaller realizable sets remain implicit by monotonicity.

\begin{theorem}\label{thm:WinM-max-correct}
For all $s\in S^\times$ and $C\subseteq \Phi$,
\[
\begin{aligned}
(s,C)\in \WinM
\;\Longleftrightarrow\;&\ 
C \text{ is realizable from } s \;
\land \\ &\neg\exists D \supset C.\; D \text{ is realizable from } s .
\end{aligned}
\]
The iteration $\WinM_i$ stabilizes after at most $|S^\times| \cdot 2^{|\Phi|}$ steps.
\end{theorem}
$\WinM$ relates each state $s$ to the $\subseteq$-maximal realizable sets of properties.
For any realizable set of properties $C\subseteq \Phi$ with $(s_0^\times,C)\in \WinM$, we can construct a winning strategy for fulfilling $C$ as a finite transducer by projecting $\WinM$ onto $C$, exactly as in the non-maximal construction in Section~\ref{sec:solving-multi-property-games}.

\begin{example}
Figure~\ref{fig:synthgame-3panel} contrasts the multi-property fixed point 
% (Fig.~\ref{fig:synthgame-a}) 
with the maximal variant. 
% (Fig.~\ref{fig:synthgame-b}).
Each state $s$ is annotated with the sets $C$ such that $(s,C)\in\WinM$.
In the multi-property fixed point~(Fig.~\ref{fig:synthgame-a}), $\WinM$ contains all realizable sets of properties.
Here, $s_1$ satisfies all goals and is therefore associated with every subset of $\{\varphi_1,\varphi_2,\varphi_3\}$, while $s_2$ (resp.\ $s_3$) is associated with all subsets of $\{\varphi_1,\varphi_2\}$ (resp.\ $\{\varphi_2,\varphi_3\}$).
At state $s_0$, output $y_1$ leads to $s_1$ or $s_2$, so only the subsets common to both successors, i.e., all subsets of $\{\varphi_1,\varphi_2\}$, are realizable; output $y_2$ leads to $s_3$ under any environment input, yielding all subsets of $\{\varphi_2,\varphi_3\}$.
Thus, $s_0$ is associated with exactly the union of these two sets.

In the maximal variant~(Fig.~\ref{fig:synthgame-b}), $\WinM$ only contains the maximal realizable sets, leaving smaller subsets implicit.
Accordingly, the sets reduce to, e.g,
$\bigl\{\{\varphi_1,\varphi_2\},\{\varphi_2,\varphi_3\}\bigr\}$ at $s_0$.
\end{example}

Our fixed-point construction avoids explicitly solving a separate synthesis game for each $C\subseteq\Phi$, but, in the worst case, our algorithm may still need to track an exponential number of goal sets for each state.
The maximal variant reduces this by retaining only the maximal goal sets. Nevertheless, the number of maximal sets can, in the worst case, still be exponential in the number of properties, as the number of $\subseteq$-maximal subsets of a set $\Phi$ with $|\Phi| = n$ is at most $\binom{n}{\lfloor n/2 \rfloor} \in \Theta(\frac{2^n}{\sqrt{n}})$~\cite{Stirling1730,Sperner28}.

\begin{theorem}\label{thm:complexity}
Multi-property $\ltlf$ realizability and synthesis are \textsc{2EXPTIME}-complete.
\end{theorem}

\noindent
This implies that neither our approach nor enumeration-based baselines can improve the worst-case complexity of the problem. On the other hand, our method shares the expensive game construction across all goal sets and propagates realizability information in a single fixed point. In the next section we further amplify this effect in a symbolic procedure.

\section{Symbolic Synthesis}
\label{sec:symbolic}

We now transition from an explicit representation of the game arena to a \emph{symbolic} representation in which states, transitions, and accepting conditions are encoded using Boolean formulas. 
This symbolic approach is grounded in Boolean synthesis~\cite{DBLP:conf/cav/FriedTV16} and yields substantial practical advantages, as fixed-point iterations can be carried out via efficient boolean operations~\cite{DBLP:conf/ijcai/ZhuTLPV17}.
We begin by reviewing symbolic DFA arenas, single-property symbolic synthesis and present our symbolic multi-property synthesis procedure.

\begin{definition}[Symbolic DFA Game Arena]\label{def:symbolic-arena}
Let $\arena=(2^{\x\cup\y},S,s_0,\delta,F)$ be a DFA game arena.
The corresponding \emph{symbolic DFA game arena} is a tuple
\[
\mathcal{F}=(\x,\y,\mathcal{Z},Z_0,\eta,f),
\]
where: $\mathcal{Z}$ is a set of $\lceil\log_2|S|\rceil$ Boolean \emph{state variables} such that each state $s\in S$ is encoded by an assignment $Z\in 2^{\mathcal{Z}}$, $Z_0\in 2^{\mathcal{Z}}$ encodes the initial state $s_0$, $\eta:\;2^{\x}\times 2^{\y}\times 2^{\mathcal{Z}}\to 2^{\mathcal{Z}}$ is a Boolean transition function such that, whenever $Z$ encodes $s$, the assignment $\eta(X,Y,Z)$ encodes the successor $\delta(s,X\cup Y)$, and $f$ is a Boolean formula over $\mathcal{Z}$ such that $f(Z)$ holds iff $Z$ encodes an accepting state, i.e., $s\in F$.
\end{definition}

\noindent
\noindent
Intuitively, $\mathcal{F}$ represents the game symbolically by viewing $X\in 2^{\x}$, $Y\in 2^{\y}$, and $Z\in 2^{\mathcal{Z}}$ as Boolean assignments to the variables in $\x$, $\y$, and $\mathcal{Z}$, with a variable set to $\ltltrue$ iff it is contained in the corresponding valuation, so that $\eta$ updates the state variables according to the transition relation, and $f$ characterizes the accepting states~\cite{DBLP:conf/ijcai/ZhuTLPV17}.

\subsection{Symbolic Single-Property Synthesis}
Following ~\cite{DBLP:conf/ijcai/ZhuTLPV17}, given a symbolic game arena $\mathcal{F}$, we maintain two kinds of formulas over the state variables: \emph{winning-state} formulas $w_i(Z)$ and  \emph{winning-move} formulas $t_i(Z,Y)$.
An assignment $Z\in 2^{\mathcal{Z}}$ encodes a winning state iff $w_i(Z)$ is $\ltltrue$, and a pair $(Z,Y)\in 2^{\mathcal{Z}}\times 2^{\y}$ encodes a winning choice of output at that state iff $t_i(Z,Y)$ is $\ltltrue$.

We initialize with the accepting condition,
$t_0(Z,Y) \equiv f(Z)$ and $w_0(Z) \equiv f(Z)$,
since the agent can terminate immediately upon reaching an accepting state.
We iteratively refine these formulas in a symbolic fixed-point computation:
\[
\begin{aligned}
t_{i+1}(Z,Y)\;&\equiv\;
t_i(Z,Y)\;\lor\\
&\quad\bigl(\neg w_i(Z) \land \forall X.\; w_i\bigl(\eta(X,Y,Z)\bigr)\bigr),\\
w_{i+1}(Z)\;&\equiv\;\exists Y.\; t_{i+1}(Z,Y).
\end{aligned}
\]
Intuitively, $t_{i+1}$ adds exactly those output choices that force the next state into the current winning region, regardless of the environment input. The conjunct $\neg w_i$ ensures that each step moves strictly closer to accepting states.
The iteration stabilizes when $w_{i+1}\equiv w_i =: w$, at which point $w$ characterizes all winning states.
The Boolean operations, can be carried out very efficiently using symbolic representations such as \emph{Binary Decision Diagrams} (BDDs)~\cite{DBLP:conf/cav/FriedTV16}.

A winning strategy can be extracted \emph{symbolically} if $w(z_0)$ is $\ltltrue$.
From the final move formula $t(Z,Y)$ we can construct a function 
$\tau:2^{\mathcal{Z}}\to 2^{\y}$ that selects, for each winning state encoding $Z$, a winning output $Y=\tau(Z)$, i.e., where $t(Z,Y)$ is $\ltltrue$~\cite{DBLP:conf/cav/FriedTV16}, inducing a symbolic transducer implementing the winning strategy~\cite{DBLP:conf/ijcai/ZhuTLPV17}.

\subsection{Symbolic Multi-Property Synthesis}
We now present our symbolic approach for multi-property $\ltlf$ synthesis.
Beyond the general benefits of symbolic methods, our approach exploits a key structural advantage: the Boolean formulas we construct naturally capture the monotonicity of multi-property realizability. Our Boolean representation of the maximal realizable sets compactly captures \emph{all} realizable subsets simultaneously without explicitly enumerating them in the formulas.
Remarkably, our symbolic procedure is algorithmically almost identical to the single-property version.
We add a linear number of Boolean variables representing the properties and then apply the same symbolic fixed-point computation as in the single-property approach. Unlike the explicit-state procedure in Section~\ref{sec:multi-prop}, which required new predecessor operators to handle sets of properties, the symbolic version does not require new operations. The result is a single formula over states and property variables that compactly captures for each state \emph{all} sets of properties that are simultaneously realizable from it.

\begin{definition}[Symbolic Multi-Property Game Arena]
\label{def:symbolic-multi-arena}
For each goal $\varphi_i \in \Phi$, let $\mathcal{F}_i = (\mathcal{X}, \mathcal{Y}, \mathcal{Z}_i, Z_i^0, \eta_i, f_i)$ be its symbolic DFA game arena. The \emph{symbolic multi-property arena}~is 
$$\mathcal{F}^{\times} = (\mathcal{X}, \mathcal{Y}, \mathcal{Z}, Z_0, \eta),$$
where $\mathcal{Z} = \mathcal{Z}_1 \cup \cdots \cup \mathcal{Z}_n$ are the state variables, $Z_0 = (Z_1^0, \ldots, Z_n^0)$ is the initial state, and $\eta = (\eta_1, \ldots, \eta_n)$ is the componentwise transition function.
\end{definition}
To represent sets of properties, we introduce $n$ Boolean variables $\mathcal{K} = \{k_1, \ldots, k_n\}$, where an assignment $K \in 2^{\mathcal{K}}$ encodes a subset $C \subseteq \Phi$ with $k_i = \ltltrue$ iff $\varphi_i \in C$.
We then define the initial winning relation formula by
\begin{equation*}
\label{eq:W0}
w_0(Z, K) := \bigwedge_{i=1}^n \bigl(k_i \Rightarrow f_i(Z_i)\bigr),
\end{equation*}
which holds when the state represented by $Z = (Z_1,\dots,Z_n) \in 2^{\mathcal{Z}}$ satisfies all properties represented by $K \in 2^{\mathcal{K}}$.
A key element of our approach is the use of implications in this formula, which directly captures the monotonicity of multi-property realizability: if $w_0(Z,K)$ holds for some set of properties, then it also holds for every subset. This allows us to represent exponentially many goal combinations compactly, without explicit enumeration.

We now apply the same symbolic fixed-point computation as in the single-property approach, with formulas ranging over both state variables $Z$ and goal variables $K$:
\begin{align*}
    t_{i+1}(Z,Y,K) &\equiv t_i(Z,Y,K) \\
    &\quad\vee \bigl(\neg w_i(Z,K) \wedge \forall X.\, w_i(\eta(X,Y,Z),K)\bigr),\\
    w_{i+1}(Z,K) &\equiv \exists Y.\, t_{i+1}(Z,Y,K),
\end{align*}
with $w_0$ defined as above and $t_0(Z,Y,K) = w_0(Z,K)$.
Crucially, this computation preserves the monotonicity property at each iteration. This allows us to apply the same symbolic operations as in the single-property case, without requiring modifications to handle the exponentially many combinations of properties. Upon stabilization, the resulting formula $w(Z,K)$ captures the full multi-property winning relation $\WinM$.
As in the single-property case, a winning strategy for a realizable $C$ can be extracted as a transducer by choosing appropriate output assignments from $t(Z,Y,K)$.

\begin{example}\label{ex:symbolic-winm}
Figure~\ref{fig:synthgame-3panel}c illustrates how the symbolic procedure compactly represents $\WinM$.
For instance, let $z_2$ be the symbolic encoding of state $s_2$ and assume it is an accepting state for $\varphi_1$ and $\varphi_2$ but not for $\varphi_3$, i.e., $f_1(z_2)=f_2(z_2)=\ltltrue$ and $f_3(z_2)=\ltlfalse$.
Then the initial winning relation formula $w_0(Z,K)$ becomes the following when fixing $Z$ to $z_2$:
\[
\begin{aligned}
w_0(z_2,K)
&\equiv (k_1\Rightarrow \ltltrue)\ \land\  (k_2\Rightarrow \ltltrue)\ \land\  (k_3\Rightarrow \ltlfalse)\\
&\equiv \neg k_3.
\end{aligned}
\]
Hence, $w_0(z_2,K)$ is satisfied by exactly those goal encodings with $k_3=\ltlfalse$, i.e., precisely the subsets of $\{\varphi_1,\varphi_2\}$.
The resulting winning relation formula for the state $s_0$ is:
\[
w(z_0,K)\equiv \neg k_1 \ \vee\ \neg k_3,
\]
which characterizes exactly the goal sets realizable from $s_0$.
This highlights the key advantage of the symbolic encoding: a small Boolean formula over $K$ succinctly captures the realizability relation over exponentially many goal combinations. In practice, simplifications are handled efficiently by the underlying BDD representation, which canonicalizes and reduces the formulas during the fixed-point computation.
\end{example}

Finally, extracting an \emph{optimal} goal set amounts to choosing a winning goal assignment $K$ such that a maximal number of goal variables $k_i$ is set to $\ltltrue$.
Given the BDD for $w(z_0,K)$, we can obtain such an assignment efficiently by quantifying over $K$ and selecting a witness that is undominated: there is no other winning assignment $K'$ that strictly extends it, i.e., sets all goals of $K$ and at least one additional goal to $\ltltrue$.

\section{Experiments}
% requires: \usepackage{booktabs}
%           \usepackage{multirow}
%           \usepackage{array} % (optional)

\begin{table}[!t]
\centering

% --- body + header font sizes (header slightly larger)
\newcommand{\tablebodyfont}{\fontsize{5.2}{5.2}\selectfont}
\newcommand{\tableheadfont}{\fontsize{5.5}{5.5}\selectfont}

\tablebodyfont
\setlength{\tabcolsep}{1.9pt}
\renewcommand{\arraystretch}{0.78}

% add extra padding just for the header row
\newcommand{\headrow}[1]{\addlinespace[1.45pt]#1\addlinespace[1.45pt]}

\resizebox{\columnwidth}{!}{%
\begin{tabular}{l @{\hspace{3pt}} c @{\hspace{3pt}} r r r}
\toprule
\headrow{\tableheadfont Family & \tableheadfont \# Goals & \tableheadfont \# States &
         \tableheadfont MPSynth [s] & \tableheadfont Enum [s] \\}
\midrule
\multirow{3}{*}{chain}
 & 2  & 16          & \textbf{0.0001} & 0.0106 \\
 & 4  & 256         & \textbf{0.0002} & 0.0196 \\
 & 6  & 15,625      & \textbf{0.0700}   & 0.157  \\
\addlinespace[2pt]

\multirow{3}{*}{counter}
 & 5  & 86,400      & \textbf{0.880}    & 2.067  \\
 & 5  & 181,200     & \textbf{2.113}    & 7.492  \\
 & 5  & 373,920     & \textbf{4.713}    & 60.74  \\
\addlinespace[2pt]

\multirow{3}{*}{until}
 & 6  & 10,077,696  & \textbf{81.10}    & 299.55 \\
 & 8  & 16,777,216  & \textbf{119.30}   & 468.29 \\
 & 10 & 60,466,176  & \textbf{562.70}   & 2640.49 \\
\addlinespace[2pt]

\multirow{3}{*}{next}
 & 7  & 35,831,808  & \textbf{109.28}   & 988.56 \\
 & 7  & 62,748,517  & \textbf{291.60}   & 2713.31 \\
 & 7  & 105,413,504 & \textbf{611.90}   & 4925.10 \\
\addlinespace[2pt]

\multirow{3}{*}{robotnav}
 & 15 & 12,582,912  & \textbf{56.50}    & 170.41 \\
 & 16 & 50,331,648  & \textbf{426.98}   & 702.01 \\
 & 17 & 100,663,296 & \textbf{1003.21}  & 2812.53 \\
\bottomrule
\end{tabular}%
}

\caption{Results across benchmark instances. For each family, we report the number of goals, arena size, and runtimes in seconds.}
\label{tab:multi-goal-co-realizability-summary}
\end{table}

\begin{figure}[t]
    \vspace*{-10pt}
    \centering
    \includegraphics[width=.86\linewidth]{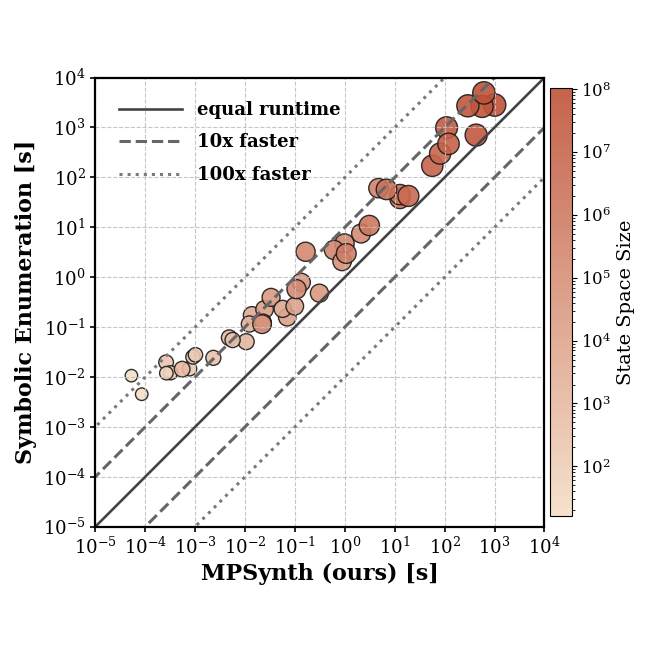}
    \vspace*{-18pt}
    \caption{Runtime comparison between our multi-property synthesis and enumeration baseline across all benchmark instances. Each point represents one instance, with color indicating state space size.}
    \label{fig:res}
\end{figure}

We implemented our symbolic multi-property synthesis procedure in a tool called \emph{MPSynth}\footnote{https://github.com/SymbolicSynthesis/MPSynth} 
based on the symbolic $\ltlf$ synthesis tool \emph{LydiaSyft}~\cite{DBLP:conf/tacas/ZhuF25}. 
Our method computes the winning relation and extracts a maximal realizable set and the corresponding winning strategy; reported runtimes include both steps.
We compare against an enumeration baseline that iterates over subsets $C\subseteq\Phi$ and checks realizability of $\bigwedge_{\varphi\in C}\varphi$ via single-property symbolic $\ltlf$ synthesis.
To avoid redundant checks, the baseline explores subsets in increasing size and prunes supersets of unrealizable sets~\cite{DBLP:conf/sigmod/AgrawalIS93}.

\paragraph{Benchmarks.}
We evaluate on a diverse set of benchmark families that stress different aspects of multi-property synthesis.
The \emph{chain}, \emph{until}, and \emph{next} families are parametric benchmarks that generate large product arenas with complex temporal dependencies~\cite{DBLP:journals/corr/abs-2511-09073}.
The \emph{counter} family extends the counter game from~\cite{DBLP:conf/ijcai/AminofGPR24}, where the agent manages a $k$-bit counter and must satisfy multiple conflicting increment policies specified by the environment.
The \emph{robotnav} family models complex robot navigation with multiple spatial objectives under environment constraints~\cite{DBLP:conf/ijcai/AminofGPR24}.
Across these families, we vary the number of properties to evaluate scalability.
Full benchmark descriptions, parameter configurations, and complete experimental results are provided in Appendix~\ref{app:exp}.

\paragraph{Results.}
Table~\ref{tab:multi-goal-co-realizability-summary} reports representative instances across all benchmark families, demonstrating the substantial performance advantage of our approach.
Figure~\ref{fig:res} visualizes the runtime comparison across all instances.
Our symbolic multi-property synthesis consistently and significantly outperforms the enumeration baseline across all problem instances, with speedups of up to two orders of magnitude.
This demonstrates that our symbolic multi-property synthesis substantially improves on the state-of-the-art enumeration-based approaches across diverse problem characteristics.

\section{Conclusion}
We presented a symbolic approach to multi-property $\ltlf$ synthesis that computes the full realizability relation between states and goal sets in a single fixed-point, avoiding explicit subset enumeration. Leveraging monotonicity, we encode goals with Boolean variables and leverage the standard symbolic fixed-point computations. Our experiments demonstrate speedups of up to two orders of magnitude over enumeration baselines, making our method an effective way to reason about goal trade-offs in over-subscribed domains.

% Our multi-property synthesis framework establishes a robust foundation that 
We intend to expand our contribution in three key directions: (i) \emph{online strategy adaptation}, leveraging our precomputed winning relation to dynamically adapt to larger goal sets becoming realizable during execution if the environment acts less adversarial; (ii) \emph{quantitative preferences}, introducing weighted properties and (iii) \emph{probabilistic environments}, extending to Markov Decision Processes by replacing the controllable predecessor with probabilistic Bellman updates to maximize the likelihood of multi-goal satisfaction.

% \section*{Ethical Statement}

% There are no ethical issues.

\section*{Acknowledgments}
This work was supported in part by the UKRI Erlangen AI Hub
on Mathematical and Computational Foundations of AI (No. EP/Y028872/1).

%% The file named.bst is a bibliography style file for BibTeX 0.99c
\bibliographystyle{named}
\bibliography{ijcai26_no_pages}

\appendix
\section{Proofs}
\label{app:proofs}
In this appendix, we provide the formal proofs for all stated theorems, establishing the correctness of our multi-property synthesis procedures, the corresponding constructions of winning strategies, and the time complexity results.

\subsection{Correctness of Multi-Property Synthesis}
\label{appendix:proof-thm1}

We provide the proof for Theorem~\ref{thm:WM-correct}, which establishes that the relation $\WinM$ characterizes multi-property realizability. We begin with some useful definitions and lemmas.

\begin{lemma}[Monotonicity of $\prem$]
\label{lem:premc-mono}
The controllable multi-property predecessor is monotonic: for any two relations $\mathcal{E}, \mathcal{E}' \subseteq S^\times \times 2^{\Phi}$, if $\mathcal{E} \subseteq \mathcal{E}'$, then $\prem(\mathcal{E}) \subseteq \prem(\mathcal{E}')$.
\end{lemma}

\begin{proof}
Let $(s, C) \in \prem(\mathcal{E})$. By definition, there exists an output $Y \in 2^{\y}$ such that for all environment responses $X \in 2^{\x}$, the successor state satisfies $(\delta^\times(s, X \cup Y), C) \in \mathcal{E}$. Since $\mathcal{E} \subseteq \mathcal{E}'$, it follows that $(\delta^\times(s, X \cup Y), C) \in \mathcal{E}'$ for all $X$, which implies $(s, C) \in \prem(\mathcal{E}')$.
\end{proof}

\begin{definition}[$i$-step Realizability]
A goal set $C \subseteq \Phi$ is \emph{$i$-step multi-goal realizable} from a state $s \in S^\times$ if there exists a strategy $g$ such that for every environment input sequence $\xi \in (2^{\x})^\omega$, the induced play $s = s_0s_1\dots$ starting at $s$ reaches a state $s_j \models C$ for some $0 \le j \le i$.
\end{definition}

\begin{lemma}[Characterization of $\WinM_i$]
\label{lem:kstepco}
For every $i \in \mathbb{N}$, state $s \in S^\times$, and goal set $C \subseteq \Phi$, 
$(s, C) \in \WinM_i$ if and only if $C$ is $i$-step multi-goal realizable from $s$.
\end{lemma}

\begin{proof}
We proceed by induction on $i$.
\paragraph{Base Case ($i=0$):} 
By definition, $(s, C) \in \WinM_0$ iff $s \models C$. This is equivalent to saying the agent can satisfy $C$ in $0$ steps by terminating immediately at $s$.

\paragraph{Inductive Step:} 
($\Rightarrow$): Let $(s, C) \in \WinM_{i+1} = \WinM_i \cup \prem(\WinM_i)$. 
If $(s, C) \in \WinM_i$, the result follows by the induction hypothesis (IH). 
Otherwise, $(s, C) \in \prem(\WinM_i)$. Thus, there exists $Y \in 2^{\y}$ such that for all $X \in 2^{\x}$, the successor $s' = \delta^\times(s, X \cup Y)$ satisfies $(s', C) \in \WinM_i$. By IH, $C$ is $i$-step realizable from every such $s'$. By playing $Y$ at state $s$, the agent ensures that $C$ is realized in at most $1+i$ steps.

($\Leftarrow$): Suppose $C$ is $(i+1)$-step realizable from $s$. If $C$ is already $i$-step realizable, then $(s, C) \in \WinM_i \subseteq \WinM_{i+1}$ by IH. 
Otherwise, there must be an output $Y$ such that for all $X$, $C$ is $i$-step realizable from $s' = \delta^\times(s, X \cup Y)$. By IH, $(s', C) \in \WinM_i$ for all $X$. By definition of the predecessor, $(s, C) \in \prem(\WinM_i)$, and hence $(s, C) \in \WinM_{i+1}$.
\end{proof}

\begin{theorem}[Restatement of Theorem 1]
The fixed point $\WinM = \bigcup_{i\ge 0} \WinM_i$ characterizes multi-property realizability: for all $s\in S^\times$ and $C\subseteq \Phi$,
\[
(s,C)\in \WinM \quad \Longleftrightarrow \quad C \text{ is realizable from } s.
\]
Furthermore, the sequence $(\WinM_i)_{i \in \mathbb{N}}$ stabilizes in at most $|S^\times| \cdot 2^{|\Phi|}$ steps.
\end{theorem}

\begin{proof}
The proof consists of two parts: stabilization and semantic correctness.

\paragraph{Stabilization:} 
The multi-property winning relation is computed as the least fixed point of the operator $F(\mathcal{E}) = \WinM_0 \cup \prem(\mathcal{E})$. Since the powerset lattice $(2^{S^\times \times 2^{\Phi}}, \subseteq)$ is complete and $\prem$ is monotonic (Lemma~\ref{lem:premc-mono}), the existence of the least fixed point is guaranteed by the \emph{Knaster-Tarski theorem}~\cite{DBLP:journals/pjm/Tarski55}. Because the domain $S^\times \times 2^{\Phi}$ is finite, the sequence  $\WinM_i$ must reach this fixed point in at most $|S^\times| \cdot 2^{|\Phi|}$.

\paragraph{Correctness:} 
We show that $(s, C) \in \WinM$ iff $C$ is realizable from $s$.
\begin{itemize}
    \item \textbf{Soundness ($\Rightarrow$):} Assume $(s, C) \in \WinM = \bigcup_{i \ge 0} \WinM_i$. Then $(s, C) \in \WinM_i$ for some finite $i$. By Lemma~\ref{lem:kstepco}, $C$ is $i$-step realizable from $s$, which implies $C$ is realizable.
    \item \textbf{Completeness ($\Leftarrow$):} If $C$ is realizable from $s$, then by the definition of $\ltlf$ synthesis (Definition~\ref{def:ltlf-synthesis}), there must exist a finite horizon $i$ such that the agent can force a win in at most $i$ steps. By Lemma~\ref{lem:kstepco}, this implies $(s, C) \in \WinM_i$. Since $\WinM_i \subseteq \WinM$, we have $(s, C) \in \WinM$.
\end{itemize}
\end{proof}

\subsection{Correctness of Strategy Extraction}
\label{appendix:proof-thm2}

We now establish the correctness of extracting a winning strategy as a finite transducer.

\begin{theorem}[Restatement of Theorem 2]
For any $C\subseteq \Phi$ with $(s_0^\times,C)\in \Win^{M}$, a transducer $T_C$ obtained by the strategy extraction construction exists and implements a winning strategy for fulfilling $C$: for every input sequence $\xi\in(2^\x)^\omega$, the play induced by $T_C$ reaches a state $s \in S^\times$ with $s\models C$ after finitely many steps, and the generated finite trace satisfies every $\varphi_i\in C$.
\end{theorem}

\begin{proof}
Fix a goal set $C \subseteq \Phi$. The multi-property winning relation restricted to $C$ is defined by $\WinM|_C = \{s \in S^\times \mid (s, C) \in \WinM\}$. Unfolding the definition of the multi-property fixed-point iteration shows that $\WinM|_C$ satisfies:
\begin{enumerate}
    \item $\WinM_0|_C = \{s \in S^\times \mid s \models C\}$,
    \item $\WinM_{i+1}|_C = \WinM_i|_C \cup \{s \in S^\times \mid \exists Y\in 2^{\y}. \forall X\in 2^{\x}. \delta^\times(s, X \cup Y) \in \WinM_i|_C\}$.
\end{enumerate}
This recovers precisely the traditional fixed-point characterization of single-property realizability played on the DFA arena $\mathcal{G}^\times$ with final states $F_C = \{s \in S^\times \mid s \models C\}$.

As shown in \cite{DBLP:conf/ijcai/GiacomoV15}, such a game solves the synthesis problem for the $\ltlf$ specification $\bigwedge_{\varphi_i \in C} \varphi_i$.
Our construction of the transducer $T_C$ coincides with the single-property case established in \cite[Theorem~4]{DBLP:conf/ijcai/GiacomoV15}. Consequently, the play reaches $F_C$ in a finite number of steps, ensuring that the generated finite trace satisfies all properties in $C$ simultaneously. 
\end{proof}

\subsection{Correctness of Maximal Multi-Property Synthesis}
\label{appendix:proof-thm3}

In this section, we establish that the maximal fixed-point iteration correctly computes the maximal realizable goal sets. In the following, we denote by $\WinM$ the multi-property realizability relation from Section~\ref{sec:solving-multi-property-games} and by $\WinMM$
the maximal variant from Section~\ref{sec:maximal_multi-prop}.

For any relation $\mathcal{E}\subseteq S^\times \times 2^{\Phi}$, we define its \emph{downward closure} as $\downcl (\mathcal{E}) := \{(s,D) \mid \exists C: (s,C)\in \mathcal{E} \text{ and } D\subseteq C\}$. This operator reconstructs all realizable subsets implied by the maximal sets.

\begin{lemma}[Downward Closure Preservation]
\label{lem:downward-closure}
For any relation $\mathcal{E}\subseteq S^\times\times 2^\Phi$, the $\maxc$ operator preserves the downward closure, i.e., $\downcl(\maxc(\mathcal{E})) = \downcl(\mathcal{E})$. Furthermore, the maximal predecessor and standard predecessor satisfy $\downcl(\premm(\mathcal{E})) = \prem(\downcl(\mathcal{E}))$.
\end{lemma}

\begin{proof}
The first identity follows immediately from the definition of $\maxc$: since $\maxc(\mathcal{E})$ only removes pairs $(s, C)$ that are subsets of some existing $(s, D) \in \mathcal{E}$, these removed subsets remain present in $\downcl(\maxc(\mathcal{E}))$.

For the second identity, fix a pair $(s, C)$. By definition, $(s, C) \in \prem(\downcl(\mathcal{E}))$ iff there exists $Y \in 2^{\y}$ such that for all $X \in 2^{\x}$, the successor state $s'$ satisfies $(s', C) \in \downcl(\mathcal{E})$. This is equivalent to $\exists D \supseteq C$ such that $(s', D) \in \mathcal{E}$. This is precisely the defining predicate for $(s, C) \in \premm(\mathcal{E})$, which, when downward closed, yields the result.
\end{proof}

\begin{theorem}[Restatement of Theorem 3]
The maximal fixed point $\WinMM = \bigcup_{i\ge 0} \WinMM_i$ characterizes the maximal realizable goal sets:
\[
\begin{aligned}
(s,C)\in \WinMM
\;\Longleftrightarrow\;&\ 
C \text{ is realizable from } s \;
\land \\ &\neg\exists D \supset C.\; D \text{ is realizable from } s .
\end{aligned}
\]
The sequence stabilizes in at most $|S^\times| \cdot 2^{|\Phi|}$ iterations. 
\end{theorem}

\begin{proof}
Let $(\WinMM_i)_{i\in\mathbb{N}}$ be the maximal iteration and consider the sequence of downward-closed relations $E_i = \downcl(\WinMM_i)$. Using the identities from Lemma~\ref{lem:downward-closure}, we have:
\begin{align*}
    E_{i+1} &= \downcl\bigl(\maxc(\WinMM_i \cup \premm(\WinMM_i))\bigr) \\
    &= \downcl(\WinMM_i) \cup \downcl(\premm(\WinMM_i)) \\
    &= E_i \cup \prem(E_i).
\end{align*}
Since $E_0 = \downcl(\maxc(\{(s, C) \mid s \models C\})) = \WinM_0$, the sequence $(E_i)_{i \in \mathbb{N}}$ is identical to the standard multi-property iteration $(\WinM_i)_{i \in \mathbb{N}}$. As established in Theorem 1, the limit of this sequence is the full realizability relation $\WinM$. 
It follows that $\downcl(\WinMM) = \WinM$. This implies:
\begin{enumerate}
    \item \textbf{Soundness:} Every $(s, C) \in \WinMM$ is contained in $\WinM$, therefore, $C$ is realizable from $s$. 
    \item \textbf{Maximality:} By the application of $\maxc$ at each iteration, no pair $(s, C)$ is maintained if a strictly larger $D \supset C$ is already shown to be realizable. Thus, $\WinMM$ contains only $\subseteq$-maximal elements.
\end{enumerate}
Conversely, if $C$ is a maximal realizable set at $s$, then $(s, C) \in \maxc(\WinM)$.Since $\WinM = \downcl(\WinMM)$, there must exist $(s, D) \in \WinMM$ with $C \subseteq D$.  Since $D$ is realizable and $C$ is maximal, it forces $C = D$, proving $(s, C) \in \WinMM$. Stabilization follows from the finiteness of the domain $S^\times \times 2^\Phi$. 
\end{proof}

\subsection{Complexity of Multi-Property Synthesis}
\label{appendix:proof-thm4}

In this section, we establish the computational complexity of multi-property $\ltlf$ realizability and synthesis. 

\begin{theorem}[Restatement of Theorem 4]
Multi-property $\ltlf$ realizability and synthesis are \textsc{2EXPTIME}-complete.
\end{theorem}

\begin{proof}
The proof follows by establishing both membership in and hardness for the \textsc{2EXPTIME} complexity class.  For a set of properties $\Phi = \{\varphi_1, \dots, \varphi_n\}$, we write $m := \sum_{i=1}^n |\varphi_i|$ and $|\mathcal{G}_i|$ for the size of an automaton $\arena_i$. We write $\poly(m)$ for an unspecified polynomial function in $m$.

\paragraph{Membership:}
The multi-property synthesis procedure consists of three primary computational phases:
\begin{enumerate}
    \item \textbf{Automata Construction:} We construct a DFA arena $\arena_i$ for each property $\varphi_i$. The construction of a DFA from an $\ltlf$ formula is at most doubly exponential in the size of the formula~\cite{DBLP:conf/ijcai/GiacomoV15}. Therefore, $|\mathcal{G}_i| \le 2^{2^{\poly(m)}}$.
    \item \textbf{Arena Construction:} The multi-property game arena is the product of the individual arenas $\mathcal{G}_i$. Its size satisfies
$|S^\times|\le (2^{2^{\poly(m)}})^n = 2^{n\cdot 2^{\poly(m)}} \le 2^{2^{\poly(m)}}$,
so the arena is at most doubly exponential in the input size.
    \item \textbf{Fixed-Point Computation:} Solving the reachability game on this arena, whether via the standard pair-based iteration ($\WinM$) or the maximal variant ($\WinMM$), takes time linear in the number of product states $|S^\times|$ and exponential in the number of properties $n$. Since $2^{|\Phi|}\le 2^m$ is only singly exponential while $|S^\times|\le 2^{2^{\poly(m)}}$ is
already doubly exponential, the fixed-point phase is bounded by $2^{2^{\poly(m)}}$ overall.
\end{enumerate}
As the overall complexity is dominated by the DFA construction step, the problem resides in \textsc{2EXPTIME}.

\paragraph{Hardness:}
The hardness is immediate from the established \textsc{2EXPTIME}-completeness of single-property $\ltlf$ synthesis~\cite{DBLP:conf/ijcai/GiacomoV15}. Any single-property synthesis instance for property $\varphi$ can be trivially reduced to a multi-property instance by setting $\Phi = \{\varphi\}$.
\end{proof}

\section{Extended Experiments}
\label{app:exp}
We provide additional details on the used benchmarks and the full experimental resutls.
All code, benchmark instances, and tool parameters are included in the supplementary material.
All experiments were carried out on an AMD EPYC CPU with 8 cores and 32GiB of RAM. 

\subsection{Benchmark Families}
\label{app:benchmarks}

We evaluate on five benchmark families that stress different aspects of multi-property $\ltlf$ synthesis.

\paragraph{\textbf{Counter.}}
The \emph{counter} family is based on the counter-game benchmark introduced in the literature.\footnote{See the benchmark description in~\cite{DBLP:conf/ijcai/AminofGPR24}.}
Intuitively, the controller maintains a $k$-bit counter and must react to environment-controlled inputs that request (potentially conflicting) counter updates over time, yielding goal sets that cannot always be jointly enforced. This creates a large symbolic arena, already for moderate $k$, and a non-trivial trade-off structure among goals.

\paragraph{\textbf{Robot navigation.}}
The \emph{robotnav} family models navigation with multiple region-reaching goals under adversarial constraints~\cite{DBLP:conf/ijcai/AminofGPR24}.
The environment can disable parts of the world, e.g., by closing doors, and reaching certain regions can irrevocably restrict access to others. As a result, satisfying one navigation objective may rule out another, making this family well-suited for evaluating maximal realizable goal sets and the strategies that achieve them.

\paragraph{\textbf{Chain, Until, Next.}}
The \emph{chain}, \emph{until}, and \emph{next} families are parametric temporal-pattern benchmarks that generate large product arenas while keeping the individual formulas syntactically simple.
They are designed to stress specific temporal constructs, e.g., long dependency chains, nested and unfolded $\ltlU$ structure, and deep $\ltlX$-progression, and are adapted from recent benchmark collections used to evaluate temporal reasoning and synthesis at scale~\cite{DBLP:journals/corr/abs-2511-09073}.
Across these families, we scale the number of properties $|\Phi|$ and the pattern parameters, i.e., formula lengths and nesting depths to obtain instances ranging from small to very large arenas.

\subsection{Full Results}
\label{app:fullresults}

Table~\ref{tab:multi-goal-co-realizability} extends Table~\ref{tab:multi-goal-co-realizability-summary} by reporting results for \emph{all} benchmark instances.
For each instance, we list the end-to-end runtime of our tool \emph{MPSynth}, which includes
(i) computing the symbolic multi-property winning relation $w(Z,K)$ via the fixed point from Section~\ref{sec:symbolic}, and
(ii) extracting a maximum-cardinality realizable set $C\subseteq\Phi$ together with a corresponding winning strategy.
 Table~\ref{tab:multi-goal-co-realizability} additionally reports the runtime spent in Step~(ii).
Finally, we include the runtime of the enumeration baseline, which includes enumerating the individual subsets, with pruning subsets that are known to be unrealizable from previous computations, and the runtime for extracting the winning strategy for a maximal realizable subset.

% requires: \usepackage{booktabs}
%           \usepackage{multirow}
%           \usepackage{makecell}

\begin{table*}[t]
\centering
\resizebox{1\linewidth}{!}{%
\begin{tabular}{lr r rrrr}
\toprule
\multicolumn{4}{c}{} & \multicolumn{2}{c}{\textbf{MPSynth}} & \multicolumn{1}{c}{} \\
\cmidrule(lr){5-6}
\thead{Family} &
\thead{Instance\\(depth / length)} &
\thead{\# Goals} &
\thead{\# States} &
\thead{Fixed Point [s]} &
\thead{MaxSAT \&\\Strategy Extraction [s]} &
\thead{Enumeration \&\\Strategy Extraction [s]} \\
\midrule

\multirow{3}{*}{chain}
 & 3 & 2 & 16      & 4.93e-05 & 3.42e-06 & 0.0106 \\
 & 3 & 4 & 256     & 2.46e-04 & 1.55e-05 & 0.0196 \\
 & 4 & 6 & 15,625  & 0.0697   & 3.08e-04 & 0.157  \\
\addlinespace

\multirow{12}{*}{until}
 & 3 & 3  & 125        & 3.10e-04 & 1.06e-05 & 0.0121 \\
 & 4 & 5  & 7,776      & 0.0136   & 8.04e-05 & 0.173  \\
 & 4 & 6  & 46,656     & 0.133    & 4.02e-04 & 0.786  \\
 & 4 & 7  & 279,936    & 0.984    & 0.0012   & 4.722  \\
 & 4 & 8  & 1,679,616  & 12.66    & 0.0024   & 37.73  \\
 & 4 & 9  & 10,077,696 & 80.34    & 0.762    & 299.55 \\
 & 4 & 10 & 60,466,176 & 557.99   & 4.705    & 2640.49 \\
 & 5 & 5  & 16,807     & 0.0245   & 6.90e-05 & 0.226  \\
 & 6 & 5  & 32,768     & 0.0331   & 6.64e-05 & 0.393  \\
 & 6 & 6  & 262,144    & 0.603    & 3.91e-04 & 3.505  \\
 & 6 & 7  & 2,097,152  & 12.74    & 9.64e-04 & 45.06  \\
 & 6 & 8  & 16,777,216 & 119.30   & 0.0027   & 468.29 \\
\addlinespace

\multirow{12}{*}{counter}
 & 2 & 3 & 168     & 7.60e-04 & 9.30e-06 & 0.0146 \\
 & 3 & 3 & 396     & 0.0023   & 1.00e-05 & 0.0241 \\
 & 4 & 3 & 912     & 0.0048   & 1.14e-05 & 0.0611 \\
 & 5 & 3 & 2,016   & 0.0122   & 1.13e-05 & 0.115  \\
 & 3 & 4 & 1,584   & 0.0105   & 2.36e-05 & 0.0508 \\
 & 4 & 4 & 3,648   & 0.0229   & 2.42e-05 & 0.124  \\
 & 5 & 4 & 8,064   & 0.0565   & 2.47e-05 & 0.231  \\
 & 2 & 5 & 18,240  & 0.0990   & 7.53e-05 & 0.261  \\
 & 3 & 5 & 40,320  & 0.309    & 9.14e-05 & 0.476  \\
 & 4 & 5 & 86,400  & 0.880    & 1.18e-04 & 2.067  \\
 & 5 & 5 & 181,200 & 2.113    & 1.49e-04 & 7.492  \\
 & 6 & 5 & 373,920 & 4.713    & 1.14e-04 & 60.74  \\
\addlinespace

\multirow{6}{*}{next}
 & 8  & 3 & 1,000       & 5.36e-04 & 1.17e-05 & 0.0143 \\
 & 10 & 5 & 248,832     & 0.163    & 1.18e-04 & 3.225  \\
 & 10 & 6 & 2,985,984   & 6.825    & 6.03e-04 & 57.37  \\
 & 10 & 7 & 35,831,808  & 109.28   & 0.0023   & 988.56 \\
 & 11 & 7 & 62,748,517  & 291.60   & 0.0029   & 2713.31 \\
 & 12 & 7 & 105,413,504 & 611.90   & 0.0329   & 4925.10 \\
\addlinespace

\multirow{8}{*}{robotnav}
 & 2 & 10 & 98,304      & 0.0168 & 0.0051 & 0.114 \\
 & 3 & 11 & 196,608     & 0.0908 & 0.0164 & 0.573 \\
 & 4 & 12 & 786,432     & 0.975  & 0.0906 & 2.975 \\
 & 5 & 13 & 1,572,864   & 2.946  & 0.145  & 10.83 \\
 & 6 & 14 & 6,291,456   & 18.28  & 0.469  & 42.38 \\
 & 7 & 15 & 12,582,912  & 55.41  & 1.085  & 170.41 \\
 & 8 & 16 & 50,331,648  & 422.12 & 4.856  & 702.01 \\
 & 9 & 17 & 100,663,296 & 993.25 & 9.962  & 2812.53 \\
\addlinespace

env          & 3 & 3 & 18  & 7.73e-05 & 7.49e-06 & 0.0045 \\
priority     & 4 & 4 & 54  & 2.45e-04 & 1.90e-05 & 0.0119 \\
resource     & 3 & 5 & 162 & 8.57e-04 & 4.50e-05 & 0.0251 \\
surveillance & 3 & 6 & 486 & 0.0055   & 1.27e-04 & 0.0549 \\
complex      & 3 & 5 & 162 & 9.66e-04 & 4.62e-05 & 0.0278 \\

\bottomrule
\end{tabular}
}
\caption{Full results across all benchmark instances. For each instance, we report the number of goals, product-arena size, and runtimes in seconds. For MPSynth, we separate the symbolic fixed-point time from the time to compute a maximal realizable property set and extract the corresponding winning strategy.}
\label{tab:multi-goal-co-realizability}
\end{table*}

\end{document}